\documentclass[10pt]{article}
\usepackage{libertine}
\usepackage[a4paper, margin=1in]{geometry}
\usepackage{amsmath, amsthm, amssymb}
\usepackage{mathtools}
\usepackage{hyperref}
\usepackage{url}
\usepackage{booktabs}
\usepackage{array}
\usepackage{xcolor}
\usepackage{graphicx}
\usepackage{enumitem}
\usepackage{caption}
\providecommand{\cref}[1]{Section~\ref{#1}}
\providecommand{\Cref}[1]{Section~\ref{#1}}

\hypersetup{
  colorlinks=true,
  linkcolor=blue!50!black,
  citecolor=blue!50!black,
  urlcolor=blue!50!black,
}

\newtheorem{theorem}{Theorem}

\newtheorem{lemma}[theorem]{Lemma}

\theoremstyle{definition}

\newtheorem{remark}[theorem]{Remark}

\newcommand{\Aff}{\mathrm{Aff}}
\newcommand{\SO}{\mathrm{SO}}
\newcommand{\SE}{\mathrm{SE}}
\newcommand{\GL}{\mathrm{GL}}
\newcommand{\ad}{\mathrm{ad}}
\newcommand{\tr}{\mathrm{tr}}
\newcommand{\Sym}{\mathrm{Sym}}
\newcommand{\algg}{\mathfrak{g}}
\newcommand{\algt}{\mathfrak{t}}
\newcommand{\algaff}{\mathfrak{aff}}
\newcommand{\algse}{\mathfrak{se}}
\newcommand{\algso}{\mathfrak{so}}
\newcommand{\alggl}{\mathfrak{gl}}
\newcommand{\R}{\mathbb{R}}

\title{The Token Is a Group Element:\\ On Lie-Algebra Attention over Matrix Lie Groups}
\author{Przemyslaw Musialski\thanks{New Jersey Institute of Technology; \texttt{przem@njit.edu}}}
\date{\today}

\begin{document}

\maketitle

\begin{abstract}
We place the attention token on the group: a token \emph{is} an element $g_i$ of a matrix Lie group $G$ --- a bare transformation, with no feature payload and no external action $\rho(g)$ carrying it. To our knowledge this is the first attention construction whose tokens are bare matrix Lie group elements. Their score is the closed-form algebra norm of the relative pose rather than a learned kernel, and it reaches the affine full-frame groups that every irrep- or surjective-exp-based method must exclude. We call it \emph{Lie-Algebra Attention}. Once tokens are group elements, the rest follows with none of the usual representation-theoretic machinery. The relative geometry of a pair is canonical, $g_i^{-1} g_j$, so the pairwise invariant $w_{ij} = \log(g_i^{-1} g_j)$ is intrinsic rather than designed. Equivariance under the diagonal $G$-action is then tautological rather than enforced, and the cocycle condition for relative poses holds automatically. The attention score is the negative squared algebra norm of the relative pose, $s_{ij} = -\|\log(g_i^{-1} g_j)\|_\lambda^2 / \tau$. This is the canonical proximity kernel on $w_{ij}$ under a block-weighted Frobenius inner product. It uses no irreducible representations, no spherical harmonics, no Clebsch--Gordan products, and no learned kernel. The construction applies to any matrix Lie group with a faithful finite-dimensional representation. We give closed-form instantiations for $\SO(2)$, $\SE(2)$, $\SO(3)$, $\SE(3)$, $\Aff(2)$, and the spatial affine extension $\Aff(3)$. The affine cases form the non-compact non-abelian regime with scale and shear that no vector-token attention method reaches: neither the irrep tradition, which has no non-trivial finite-dimensional unitary irreps, nor the surjective-exp tradition (LieConv-style lifting). Three sequence-completion experiments, on $\SE(2)$, $\SO(3)$, and $\Aff(2)$, bear this out. The closed-form score matches a learned MLP kernel on the same invariant and outperforms it on $\SE(2)$, using $50$ to $80\times$ fewer score parameters; the learned kernel does not improve on it. A vector-token baseline, by contrast, breaks invariance by five to twelve orders of magnitude. Prior group-element attention (LieTransformer) attaches a feature vector to each token and learns the kernel; the inversion here is to keep the token a bare group element and read the score off in closed form.
\end{abstract}

\section{Introduction}
\label{sec:intro}

This paper changes what an attention token \emph{is}. In equivariant models for spatial reasoning (point clouds, molecules, robotic poses, 2D layouts, protein structure), a token is almost always a feature vector $v \in V$, and the symmetry group $G$ acts on it \emph{externally} through a representation $\rho(g): V \to V$. Equivariance is then something to \emph{enforce}, through machinery built for that purpose: irreducible representations and Clebsch--Gordan tensor products (Tensor Field Networks~\cite{thomas2018tfn}, the $\SE(3)$-Transformer~\cite{fuchs2020se3}, Equiformer~\cite{liao2022equiformer}), steerable kernels~\cite{weiler2019general}, multivector sandwich products (the Geometric Algebra Transformer~\cite{brehmer2023gatr}), or auxiliary frames (AlphaFold's Invariant Point Attention~\cite{jumper2021alphafold}). In every such case the token \emph{carries} a transformation action; it is not itself a transformation. We make the token a transformation instead: it \emph{is} an element $g_i \in G$ of a matrix Lie group --- the group element is the token, carrying nothing else.

Spatial reasoning then operates on transformations directly, rather than on vector embeddings that carry transformation actions, collapsing the standard distinction between the data and the symmetry acting on it. The consequences are immediate and structural. The relative geometry of two tokens is canonical, $g_i^{-1} g_j$, and the pairwise invariant $w_{ij} = \log(g_i^{-1} g_j) \in \algg$ is intrinsic, not designed. Equivariance of the output is tautological rather than trained. Consistency of all pairwise relative poses, the cocycle condition, is automatic. None of the representation-theoretic machinery above is required.

With tokens on the group, the attention score is read off in closed form. The natural diagonal-invariant scalar of a token pair is a function of $w_{ij} = \log(g_i^{-1} g_j)$; because $w_{ij}$ is itself invariant, any fixed bilinear form on $\algg$ applied to it is invariant too, and the block-weighted squared norm $s_{ij} = -\|w_{ij}\|_\lambda^2 / \tau$ is the canonical proximity choice among them --- a closed-form kernel with no irreducible representations, no spherical harmonics, no Clebsch--Gordan products, and no learned kernel network. Equivariance and the cocycle follow as theorems with one-line proofs (\cref{sec:construction}). The simplicity is not an implementation shortcut; it is what the representation-theoretic framing of equivariance has obscured. We call the resulting construction \emph{Lie-Algebra Attention}.

We make three contributions, one per pillar of the construction:
\begin{itemize}[leftmargin=*,topsep=2pt,itemsep=0pt]
\item \textbf{The bare group-element token.} The token \emph{is} an element of a matrix Lie group --- a transformation with no feature payload and no external action $\rho(g)$ --- rather than a vector the group acts on. Equivariance, the intrinsic pairwise invariant $w_{ij} = \log(g_i^{-1} g_j)$, and cocycle consistency follow from the group structure with one-line proofs (\cref{sec:construction}); a vector-token control that reintroduces a payload breaks equivariance by five to twelve orders of magnitude (\cref{sec:experiments}).
\item \textbf{The closed-form canonical score.} The score is the negative squared algebra norm of the relative pose, $s_{ij} = -\|w_{ij}\|_\lambda^2 / \tau$ --- the canonical proximity kernel on the invariant, read off in closed form with no irreducible representations, Clebsch--Gordan products, steerable kernels, or a learned kernel network. A learned MLP kernel on the same invariant, carrying $50$ to $80\times$ more score parameters, does not improve on it (\cref{sec:experiments}).
\item \textbf{The affine full-frame regime.} Closed-form instantiations for six groups --- $\SO(2)$, $\SE(2)$, $\SO(3)$, $\SE(3)$, $\Aff(2)$, $\Aff(3)$ --- including the non-compact non-abelian affine cases with scale and shear. These lie beyond the irrep tradition (no non-trivial finite-dimensional unitary irreps) and the surjective-exp tradition (LieConv); $\Aff(2)$ is validated empirically (\cref{sec:exp:aff2}) and $\Aff(3)$ is given in closed form as the spatial affine analogue, with empirical validation deferred to follow-up work (\cref{sec:instantiations}).
\end{itemize}

Every prior tradition keeps the token a vector and the group action external, and the limitation each inherits is a symptom of that shared ontology: the irrep tradition is restricted by unitarity to compact groups; geometric algebra keeps tokens in flat ambient space, off the group manifold; capsule networks~\cite{hinton2018matrix} use unconstrained pose matrices with no Lie-algebra structure; AlphaFold IPA gets the squared-distance kernel right for the position part but handles rotation separately via frame transformation; LieTransformer~\cite{hutchinson2021lietransformer} attends over $(g, v)$ pairs and learns the kernel, and the LieConv~\cite{finzi2020lieconv} lifting it builds on requires a surjective exponential, excluding $\Aff(n)$; RoPE~\cite{su2021rope} uses $\SO(2)$ rotations only for positional encoding. Placing the token strictly on the group is what removes the restriction in each case. A detailed comparison is in \cref{sec:related}.

The paper is organized as follows. \cref{sec:construction} develops the construction for any matrix Lie group, using $\Aff(2)$ as the running example. \cref{sec:instantiations} specializes it to the five remaining groups $\SO(2)$, $\SE(2)$, $\SO(3)$, $\SE(3)$, and the spatial affine $\Aff(3)$ ($\Aff(2)$ being the running example). \cref{sec:architecture} gives the transformer architecture in concrete terms. \cref{sec:experiments} validates it on $\SE(2)$, $\SO(3)$, and $\Aff(2)$. \cref{sec:related} positions the construction against prior work. \cref{sec:discussion} discusses limitations; \cref{sec:conclusion} concludes.

\section{The Construction}
\label{sec:construction}

This section develops the construction at the level of any matrix Lie group $G$ with a faithful finite-dimensional representation, using $\Aff(2)$ as the running concrete example. The premise to keep in front of every step is the one from \cref{sec:intro}: the token \emph{is} an element of $G$, not a vector that $G$ acts on. Three properties of any matrix Lie group are all we use: (i) closed-form exponential and logarithm on a principal chart $\mathcal{U} \subset G$, (ii) the cancellation identity $(a g_i)^{-1}(a g_j) = g_i^{-1} g_j$, and (iii) a block decomposition of the Lie algebra under the natural orthogonal action of $O(n)$. From these the invariant, the score, equivariance, and the cocycle all follow as consequences. No representation-theoretic machinery enters anywhere.

\subsection{Setting: tokens as group elements}
\label{sec:setting}

A token in this construction is an element $g \in G$ of a matrix Lie group --- a transformation in its own right, not a vector in a space on which $G$ acts. The standard equivariant-ML setup has a token $v \in V$ and a representation $\rho(g): V \to V$ acting on it; we discard $V$ and $\rho$ entirely. For the running example $G = \Aff(2)$ --- the group of invertible affine maps of the plane --- a token is a \emph{framed point}: an origin together with two basis vectors. As a homogeneous matrix,
\[
\Aff(2) = \left\{ \begin{pmatrix} A & t \\ 0 & 1 \end{pmatrix} : A \in \GL(2, \R),\; t \in \R^2 \right\},
\]
where $A$ encodes rotation, scale, and shear of the frame, $t$ encodes the origin, and the group has six degrees of freedom. Group composition and inversion are closed-form:
\[
(A_1, t_1) \cdot (A_2, t_2) = (A_1 A_2,\; A_1 t_2 + t_1), \qquad (A, t)^{-1} = (A^{-1},\; -A^{-1} t).
\]

For general matrix $G$, multiplication is matrix multiplication of the representation, and inversion is matrix inversion. No further structure is required.

\subsection{Block decomposition of the algebra}
\label{sec:block-decomposition}

The Lie algebra $\algg$ of $G$ is the tangent space at the identity, and the bracket is the matrix commutator $[X, Y] = XY - YX$. For $G = \Aff(2)$,
\[
\algaff(2) = \left\{ \begin{pmatrix} X & v \\ 0 & 0 \end{pmatrix} : X \in \alggl(2, \R),\; v \in \R^2 \right\} \cong \R^6,
\]
with bracket $[(X_1, v_1), (X_2, v_2)] = ([X_1, X_2],\; X_1 v_2 - X_2 v_1)$. The non-trivial bracket reflects the non-abelian interaction between linear part and translation.

The algebra decomposes into geometrically distinct subspaces under the natural orthogonal action on the matrix representation. For $\algaff(2)$ the decomposition has four irreducible blocks:
\begin{equation}
\label{eq:aff2-blocks}
\algaff(2) = \underbrace{\R^2}_{\text{translation}} \oplus \underbrace{\algso(2)}_{\text{rotation}} \oplus \underbrace{\R}_{\text{iso.\ scale}} \oplus \underbrace{\Sym_0(2)}_{\text{aniso.\ + shear}},
\end{equation}
with dimensions $2 + 1 + 1 + 2 = 6$. Each block carries a distinct geometric meaning and is invariant under the relevant orthogonal action. For other groups the decomposition is simpler: $\algso(3)$ is one block, $\algse(n)$ splits into translation + rotation. The full list is in Table~\ref{tab:instantiations}.

\textbf{Orthonormal basis.} Each block has a natural matrix representative. The unnormalized generators of the linear-part blocks each have Frobenius norm $\sqrt{2}$; to obtain an orthonormal basis under the Frobenius inner product (\cref{sec:inner-product}), we divide each by $\sqrt{2}$. For $\Aff(2)$, the translation generators
\[
T_x = \begin{pmatrix} 0 & 0 & 1 \\ 0 & 0 & 0 \\ 0 & 0 & 0 \end{pmatrix}, \qquad
T_y = \begin{pmatrix} 0 & 0 & 0 \\ 0 & 0 & 1 \\ 0 & 0 & 0 \end{pmatrix}
\]
are already orthonormal, and the linear-part generators, normalized by $1/\sqrt{2}$, are
\[
R = \frac{1}{\sqrt{2}} \begin{pmatrix} 0 & -1 \\ 1 & \phantom{-}0 \end{pmatrix}, \quad
S = \frac{1}{\sqrt{2}} \begin{pmatrix} 1 & 0 \\ 0 & 1 \end{pmatrix}, \quad
D_1 = \frac{1}{\sqrt{2}} \begin{pmatrix} 1 & 0 \\ 0 & -1 \end{pmatrix}, \quad
D_2 = \frac{1}{\sqrt{2}} \begin{pmatrix} 0 & 1 \\ 1 & 0 \end{pmatrix}.
\]
All six basis elements satisfy $\tr(B_i^\top B_j) = \delta_{ij}$. Every element of $\algaff(2)$ is written
\[
W = t_x T_x + t_y T_y + \theta R + s S + q_1 D_1 + q_2 D_2,
\]
with coordinates $(t_x, t_y, \theta, s, q_1, q_2) \in \R^6$.

The orthonormal coordinates absorb the $\sqrt{2}$ normalization. If $\varphi$ is the physical rotation angle (so that the unnormalized rotation block is $\varphi$ times the antisymmetric matrix in $R$), then $\theta = \sqrt{2}\, \varphi$. Similarly, if $\sigma$ is the physical isotropic-scale factor, then $s = \sqrt{2}\, \sigma$. The translation coordinates $t_x, t_y$ are unchanged.

\subsection{Exponential and logarithm}
\label{sec:exp-log}

The exponential map sends algebra elements to group elements. For algebras with a block-triangular structure $(X, v) \mapsto \left( \begin{smallmatrix} X & v \\ 0 & 0 \end{smallmatrix} \right)$ --- which covers $\Aff(2)$, $\SE(2)$, $\SE(3)$ --- the exponential has the block closed form
\[
\exp(X, v) = \bigl(e^X,\; V(X)\, v\bigr), \qquad V(X) = \int_0^1 e^{rX}\, dr = \sum_{k \ge 0} \frac{X^k}{(k+1)!}.
\]
When $X$ is invertible, $V(X) = X^{-1}(e^X - I)$.

For $\Aff(2)$, $X \in \alggl(2, \R)$ is a $2 \times 2$ matrix and Cayley--Hamilton gives
\[
X^2 - (\tr X)\, X + (\det X)\, I = 0,
\]
so every higher power $X^k$ with $k \ge 2$ reduces to a linear combination of $I$ and $X$. Summing the exponential series collapses to
\[
e^X = \alpha(X)\, I + \beta(X)\, X,
\]
where $\alpha, \beta$ are scalar functions of $\tr X$ and $\det X$. The same block-triangular form gives closed-form $V(\omega)$ via the Rodrigues route for $\SE(2)$ and $\SE(3)$ (see \cref{sec:instantiations}).

The principal logarithm inverts $\exp$ on a chart $\mathcal{U} \subset G$ --- an open neighborhood of the identity on which $\log$ is single-valued. For $\Aff(2)$,
\[
\mathcal{U}_{\Aff(2)} = \bigl\{(A, t) : A \text{ has no eigenvalue on } \R_{\le 0}\bigr\},
\]
and on this chart
\[
\log(A, t) = \bigl(\log A,\; V(\log A)^{-1}\, t\bigr).
\]
The maps $\exp$ and $\log$ are smooth on their domains and are each other's inverses on $\mathcal{U}$. They provide the bridge between $G$, where tokens live, and $\algg$, where standard linear operations apply. For the compact groups $\SO(2), \SO(3)$, the chart is $\theta \in (-\pi, \pi)$ for the rotation angle. For $\SE(n)$ the chart inherits from the rotation part. The full per-group specification is in \cref{sec:instantiations}.

\subsection{Block-weighted inner product}
\label{sec:inner-product}

The algebra $\algg$ sits inside $\alggl(m, \R)$ for some $m$ via the matrix representation. The Frobenius inner product on $\alggl(m, \R)$ restricts to a positive-definite inner product on $\algg$:
\[
\langle W_1, W_2 \rangle = \tr(W_1^\top W_2).
\]
In the orthonormal basis of \cref{sec:block-decomposition}, for $\Aff(2)$ this takes the diagonal form
\[
\|W\|^2 = t_x^2 + t_y^2 + \theta^2 + s^2 + q_1^2 + q_2^2.
\]

\textbf{Weighted form.} Respecting the block structure and the isotropy within each block, the most general $O(n)$-invariant block-diagonal positive-definite inner product is
\begin{equation}
\label{eq:weighted-norm}
\|W\|_\lambda^2 = \lambda_t (t_x^2 + t_y^2) + \lambda_\theta\, \theta^2 + \lambda_s\, s^2 + \lambda_q\, (q_1^2 + q_2^2),
\end{equation}
with positive weights $\lambda = (\lambda_t, \lambda_\theta, \lambda_s, \lambda_q) \in \R_{>0}^4$ for $\Aff(2)$. The Frobenius form is $\lambda = (1, 1, 1, 1)$. For other groups the weight count drops: two weights for $\SE(n)$, one weight for $\SO(n)$. The full count per group is in Table~\ref{tab:instantiations}.

\begin{remark}[No canonical Ad-invariant metric on $\algaff(2)$]
\label{rem:no-ad-invariant}
The algebra $\algaff(2)$ is not semisimple; it contains the nonzero translation ideal $\algt = \R^2$. We show no Ad-invariant positive-definite metric exists. Let $B$ be any Ad-invariant symmetric bilinear form. Take the isotropic-scale generator $sI$ for $s \neq 0$; for $v_1, v_2 \in \algt$, Ad-invariance gives
\[
B(s v_1, v_2) + B(v_1, s v_2) = 0 \;\Longrightarrow\; 2 s\, B(v_1, v_2) = 0,
\]
hence $B|_{\algt \times \algt} = 0$. The cross-pairing $B|_{\algt \times \alggl(2)}$ vanishes by the same generator: for $v \in \algt$ and $Y \in \alggl(2)$, $[sI, v] = s\, v$ while $[sI, Y] = 0$ since $sI$ is central in $\alggl(2)$, so Ad-invariance gives $s\, B(v, Y) = 0$, hence $B(v, Y) = 0$. Therefore $\algt$ lies in the radical of every Ad-invariant symmetric bilinear form, and no such form is non-degenerate.
\end{remark}

The Frobenius form is non-degenerate and positive-definite but \emph{not} Ad-invariant on $\algaff(2)$. The construction does not require Ad-invariance of the metric: it requires that $w_{ij}$ itself be invariant (Lemma~\ref{lem:invariance}), so any fixed bilinear form on $\algg$ applied to $w_{ij}$ yields an invariant scalar. We use the Frobenius form because it is positive-definite, gives equal weight to all coordinates in the orthonormal basis of \cref{sec:block-decomposition}, and admits the block-weighted refinement \eqref{eq:weighted-norm}. For $\SE(2), \SE(3)$, an analogous no-Ad-invariant-metric result holds via a structurally different argument (using translation generators acting on rotation through the bracket); see Remarks in \cref{sec:instantiations}. For $\SO(3)$ (and compact simple algebras generally) the Killing form is non-degenerate and the metric is canonical up to a positive scalar. For $\SO(2)$ the algebra is abelian and the Killing form vanishes identically; any positive scalar gives an Ad-invariant metric on the 1-dimensional algebra.

\subsection{The invariant \texorpdfstring{$w_{ij}$}{w\_ij}}
\label{sec:invariant}

With tokens on the group, the pairwise invariant is not something we have to design. It is intrinsic to the token type. Given two tokens $g_i, g_j \in G$, their \emph{relative pose} is the group element $g_i^{-1} g_j$. For $\Aff(2)$:
\[
g_i^{-1} g_j = \bigl( A_i^{-1} A_j,\; A_i^{-1}(t_j - t_i) \bigr).
\]
When $g_i^{-1} g_j \in \mathcal{U}$, mapping to the algebra gives a $\dim(\algg)$-vector encoding how $j$ differs from $i$ in each geometric component:
\[
w_{ij} = \log\bigl(g_i^{-1} g_j\bigr) \in \algg.
\]

\begin{lemma}[Diagonal invariance]
\label{lem:invariance}
$w_{ij}$ is invariant under the diagonal action $g_i \mapsto a g_i$ for every $a \in G$.
\end{lemma}
\begin{proof}
$(a g_i)^{-1}(a g_j) = g_i^{-1}\, a^{-1}\, a\, g_j = g_i^{-1}\, g_j$, and $\log$ depends only on the group element.
\end{proof}

The cancellation needs only the group axioms. It holds for every matrix Lie group, and so does the lemma.

\begin{remark}[Antisymmetry]
On the chart, $\log(g^{-1}) = -\log(g)$, so $w_{ji} = -w_{ij}$ and $\|w_{ji}\|_\lambda^2 = \|w_{ij}\|_\lambda^2$. The symmetry is chart-local; it requires both $g_{ij}$ and $g_{ij}^{-1}$ to lie in $\mathcal{U}$.
\end{remark}

For $\Aff(2)$, the coordinates of $w_{ij}$ in the basis of \cref{sec:block-decomposition} are the relative translation, relative rotation, relative isotropic scale, and relative anisotropic-scale-and-shear between $g_i$ and $g_j$. Each of the four block coordinates is itself $G$-invariant.

\subsection{The score \texorpdfstring{$s_{ij} = -\|w_{ij}\|_\lambda^2 / \tau$}{s\_ij = -||w\_ij||\^2 / tau}}
\label{sec:score}

Once $w_{ij}$ is the intrinsic invariant of a token pair, the score follows as its canonical readout. The natural diagonal-invariant scalar built from $w_{ij}$ is its squared norm under the block-weighted inner product of \cref{sec:inner-product}. Because $w_{ij}$ is itself invariant (Lemma~\ref{lem:invariance}), any fixed bilinear form on $\algg$ applied to it is invariant as well. The block-weighted squared norm is therefore the canonical proximity readout, with the block-diagonal $O(n)$-isotropic shape fixed in \cref{sec:inner-product}. The word \emph{canonical} carries two senses here. The functional form, a squared algebra norm of the invariant, is canonical given $w_{ij}$. The block weights $\lambda$ are a minimal learned refinement, since no canonical Ad-invariant metric exists on the non-compact cases (Remark~\ref{rem:no-ad-invariant}). \cref{sec:experiments} finds that a learned kernel on the same invariant does not improve on this form. The attention score between tokens $i$ and $j$ is the negative of that, divided by a temperature:
\begin{equation}
\label{eq:score}
\boxed{\; s_{ij} = -\|w_{ij}\|_\lambda^2 \,/\, \tau, \qquad \tau \in \R_{>0}. \;}
\end{equation}
This is a proximity kernel: tokens with similar poses receive high scores; tokens with large geometric differences receive low scores. Each of the block weights controls how strongly one type of geometric difference contributes.

By Lemma~\ref{lem:invariance} the score is invariant under the diagonal $G$ action for every choice of $\lambda$ and $\tau$, since $w_{ij}$ is invariant and the norm is a deterministic function of $w_{ij}$.

\textbf{Multi-head structure.} With $H$ heads, each head $k = 1, \dots, H$ carries its own weight vector $\lambda_k \in \R_{>0}^K$, where $K$ is the number of blocks in the algebra (Table~\ref{tab:instantiations}). With a per-head temperature $\tau_k$, the total number of score parameters is $H(K + 1)$: for $\Aff(2)$ ($K = 4$), $5H$ parameters; for $\SE(n)$ ($K = 2$), $3H$; for $\SO(n)$ ($K = 1$), $2H$. Only the ratios $\lambda_k / \tau_k$ matter; $\tau_k$ can be absorbed into $\lambda_k$. One head may weight translation heavily, attending by spatial proximity. Another may weight rotation, attending by orientation similarity. The block decomposition of \cref{sec:block-decomposition} is what makes this separation possible.

\textbf{Domain.} The score is defined for token pairs where $g_i^{-1} g_j \in \mathcal{U}$. On or near the chart boundary, the global form of \cref{sec:equivariant-output} applies but the closed-form score does not.

\subsection{Equivariant output (global and local form)}
\label{sec:equivariant-output}

Equivariance, in this construction, is not a property to be enforced by representation machinery. It is tautological: the form of every equivariant map is fixed by the token type. The following theorem states it.

\begin{theorem}[Equivariant output, global form]
\label{thm:equivariant-output}
Every map $\{g_i\} \mapsto \{\hat{g}_i\}$ satisfying $\hat{g}_i(a \cdot) = a \cdot \hat{g}_i(\cdot)$ for all $a \in G$ has the form
\begin{equation}
\label{eq:output-global}
\hat{g}_i = g_i\, \Delta_i, \qquad \Delta_i \in G \text{ invariant under the diagonal } G \text{ action.}
\end{equation}
\end{theorem}

\begin{proof}
Define $\Delta_i = g_i^{-1}\, \hat{g}_i$. Under $g_i \mapsto a g_i$ and $\hat{g}_i \mapsto a \hat{g}_i$:
\[
g_i^{-1}\, \hat{g}_i \;\mapsto\; (a g_i)^{-1}(a \hat{g}_i) = g_i^{-1}\, \hat{g}_i,
\]
so $\Delta_i$ is invariant. The reconstruction $\hat{g}_i = g_i \Delta_i$ recovers the original map.
\end{proof}

\textbf{Local form.} When $\Delta_i \in \mathcal{U}$, write
\[
\Delta_i = \exp(\delta_i),\qquad \delta_i = \log(\Delta_i) \in \algg.
\]
Since $\Delta_i$ is invariant, so is $\delta_i$. The output becomes
\begin{equation}
\label{eq:output-local}
\hat{g}_i = g_i\, \exp(\delta_i).
\end{equation}
This is the form used in practice: the model predicts $\delta_i \in \R^{\dim \algg}$ from invariant features and the output pose is $g_i \exp(\delta_i)$.

The exponential map of $G$ is not surjective in general: not every group element is the exponential of an algebra element. The local form \eqref{eq:output-local} covers tokens whose corrections lie on the chart; the global form \eqref{eq:output-global} applies to the rest.

\subsection{Cocycle preservation}
\label{sec:cocycle}

The relative poses between output tokens must be internally consistent: $\hat{g}_{ij}\, \hat{g}_{jk}$ must equal $\hat{g}_{ik}$ for all triples $i, j, k$. This is the cocycle condition. Because the tokens \emph{are} group elements, the output formula \eqref{eq:output-global} satisfies it automatically. The model cannot emit a globally inconsistent set of poses, even adversarially. The consistency lives in the token type, not in the loss.

For $\hat{g}_i = g_i \Delta_i$, the output relative pose is
\[
\hat{g}_{ij} = \hat{g}_i^{-1}\, \hat{g}_j = \Delta_i^{-1}\, g_{ij}\, \Delta_j.
\]
Composing two consecutive relative poses:
\[
\hat{g}_{ij}\, \hat{g}_{jk} = \Delta_i^{-1}\, g_{ij}\, \underbrace{\Delta_j\, \Delta_j^{-1}}_{=\,e}\, g_{jk}\, \Delta_k = \Delta_i^{-1}\, g_{ik}\, \Delta_k = \hat{g}_{ik}.
\]
The local form $\hat{g}_i = g_i \exp(\delta_i)$ gives the same conclusion: $\hat{g}_{ij}\, \hat{g}_{jk} = \exp(-\delta_i)\, g_{ik}\, \exp(\delta_k) = \hat{g}_{ik}$.

The cocycle holds for every choice of $\Delta_i$ (or $\delta_i$), with no constraint imposed on the model. Consistency is structural, not learned.

\section{Six Instantiations}
\label{sec:instantiations}

The construction of \cref{sec:construction} applies to any matrix Lie group with a faithful finite-dimensional representation. This section instantiates it explicitly for $\SO(2)$, $\SE(2)$, $\SO(3)$, $\SE(3)$, and $\Aff(3)$; the remaining group, $\Aff(2)$, is the running worked example throughout \cref{sec:construction}. The compact and Euclidean cases $\SO(2)$, $\SE(2)$, $\SO(3)$, $\SE(3)$ are all instances of one token type. Together they reproduce the regimes that prior methods handle one group at a time. The affine cases $\Aff(2)$ and $\Aff(3)$ are non-compact and non-abelian, with scale and shear. No irrep method can reach them, since they have no non-trivial finite-dimensional unitary irreps, and no surjective-exp method can either, since the LieConv-style group lifting requires a surjective exponential that the affine cases lack. Yet they fall out of the same construction with only a larger block count. They are the point of the section. The six instantiations cover all groups relevant to 2D and 3D spatial reasoning.

The proofs of invariance, equivariant output, and cocycle preservation in \cref{sec:construction} apply to all six instantiations without modification. They use only the group axioms and the cancellation $(ag_i)^{-1}(ag_j) = g_i^{-1}g_j$.

\subsection{\texorpdfstring{$\SO(2)$}{SO(2)} --- Planar rotations}
\label{sec:so2}

\textbf{Group.} $\SO(2) = \{R(\theta) : \theta \in \R\}$, dimension 1.

\textbf{Algebra.} $\algso(2) = \{\omega J : \omega \in \R\}$ where $J = \begin{pmatrix} 0 & -1 \\ 1 & 0 \end{pmatrix}$. Dimension 1.

\textbf{Exp/log.} $\exp(\omega J) = R(\omega)$; $\log(R(\theta)) = \theta J$ for $\theta \in (-\pi, \pi)$.

\textbf{Block decomposition.} One block: rotation. $\algso(2) = \underbrace{\R}_{\text{rotation}}$.

\textbf{Orthonormal basis.} $R_0 = \frac{1}{\sqrt{2}} J$. Coordinate: $\delta\theta = \sqrt{2}\,\omega$.

\textbf{Score.} One weight:
\[
s_{ij} = -\lambda_\theta\, \delta\theta_{ij}^2 / \tau.
\]

Block-weight count per head: 1 (plus temperature). The score operates on the same invariant as RoPE (the angular difference) but with quadratic, not oscillatory, dependence.

\subsection{\texorpdfstring{$\SE(2)$}{SE(2)} --- Rigid motions of the plane}
\label{sec:se2}

\textbf{Group.} $\SE(2) = \SO(2) \ltimes \R^2$, dimension 3.

\textbf{Algebra.}
\[
\algse(2) = \left\{ \begin{pmatrix} \omega J & v \\ 0 & 0 \end{pmatrix} : \omega \in \R,\; v \in \R^2 \right\} \cong \R^3.
\]
Bracket: $[(\omega_1, v_1), (\omega_2, v_2)] = (0,\; \omega_1 J v_2 - \omega_2 J v_1)$.

\textbf{Exp/log.} Closed-form via
$V(\omega) = \frac{\sin\omega}{\omega} I + \frac{1-\cos\omega}{\omega} J$
and
$V(\theta)^{-1} = \frac{\theta}{2}\cot(\theta/2)\, I - \frac{\theta}{2}\, J$, on the chart $\theta \in (-\pi, \pi)$.

\textbf{Block decomposition.} Two blocks:
\[
\algse(2) = \underbrace{\R^2}_{\text{translation}} \oplus \underbrace{\R}_{\text{rotation}}.
\]

\textbf{Orthonormal basis.} $T_x, T_y$ (translation, unit norm), $R_0 = \frac{1}{\sqrt{2}} J$ (rotation, normalized).

\textbf{Score.} Two weights:
\[
s_{ij} = -\bigl(\lambda_t (\delta t_x^2 + \delta t_y^2) + \lambda_\theta\, \delta\theta^2\bigr) / \tau.
\]

Block-weight count per head: 2 (plus temperature). Experimentally validated in \cref{sec:experiments}.

\begin{remark}[No Ad-invariant positive-definite metric on $\algse(2)$]
The translation ideal $\algt = \R^2$ is a nonzero abelian ideal of $\algse(2)$. Let $B$ be any Ad-invariant symmetric bilinear form. For $X = (0, v_X)$, $Y = (0, v_Y) \in \algt$ and $Z = (\omega, 0) \in \algso(2)$ with $\omega \neq 0$: since translations commute, $[X, Y] = 0$, so Ad-invariance gives $B(Y, [X, Z]) = 0$. The bracket evaluates to $[X, Z] = (0, -\omega J v_X) \in \algt$. Since $J$ is invertible, $\{J v_X : v_X \in \R^2\} = \R^2$. Therefore $B|_{\algt \times \algt} = 0$. A positive-definite form cannot vanish on a non-zero subspace, so no Ad-invariant positive-definite metric exists. The Frobenius form is non-degenerate and positive-definite but not Ad-invariant.
\end{remark}

\subsection{\texorpdfstring{$\SO(3)$}{SO(3)} --- Spatial rotations}
\label{sec:so3}

\textbf{Group.} $\SO(3) = \{R \in \R^{3 \times 3} : R^\top R = I,\; \det R = 1\}$, dimension 3.

\textbf{Algebra.} $\algso(3) = \{X \in \R^{3 \times 3} : X^\top = -X\}$, spanned by
\[
L_x = \begin{pmatrix} 0 & 0 & 0 \\ 0 & 0 & -1 \\ 0 & 1 & 0 \end{pmatrix},\quad
L_y = \begin{pmatrix} 0 & 0 & 1 \\ 0 & 0 & 0 \\ -1 & 0 & 0 \end{pmatrix},\quad
L_z = \begin{pmatrix} 0 & -1 & 0 \\ 1 & 0 & 0 \\ 0 & 0 & 0 \end{pmatrix}.
\]
Each has $\|L_\alpha\|_F = \sqrt{2}$. Bracket: $[L_x, L_y] = L_z$ and cyclic.

\textbf{Exp/log.} Rodrigues formula:
\[
\exp(\theta\, \hat{n}_\times) = I + \sin\theta\, \hat{n}_\times + (1 - \cos\theta)\, \hat{n}_\times^2,
\]
where $\hat{n} \in S^2$ is the axis and $\theta \in [0, \pi)$. Log on the chart $\theta \in (0, \pi)$:
\[
\theta = \arccos\!\left(\frac{\tr(R) - 1}{2}\right),\qquad
\hat{n}_\times = \frac{R - R^\top}{2 \sin\theta}.
\]

\textbf{Block decomposition.} One block: rotation. $\algso(3) = \underbrace{\R^3}_{\text{rotation}}$. The algebra is simple; it does not decompose further.

\textbf{Orthonormal basis.} $R_\alpha = \frac{1}{\sqrt{2}} L_\alpha$ for $\alpha \in \{x, y, z\}$. Coordinates: $\delta\theta_\alpha = \sqrt{2}\, \omega_\alpha$.

\textbf{Score.} One weight:
\[
s_{ij} = -\lambda_\theta (\delta\theta_x^2 + \delta\theta_y^2 + \delta\theta_z^2) / \tau.
\]

Block-weight count per head: 1 (plus temperature). This is the squared geodesic distance on $\SO(3)$, up to a positive scalar.

\begin{remark}[Canonical Ad-invariant metric on $\algso(3)$]
$\algso(3)$ is simple and compact. The Killing form $K(X, Y) = \tr(\ad_X \ad_Y)$ is negative-definite and Ad-invariant. Computed in the standard 3D representation:
$K(L_x, L_x) = \tr(\ad_{L_x}^2) = -2 = \tr(L_x^2)$, so $K(X, Y) = \tr(XY)$. For antisymmetric matrices, $\tr(X Y) = -\tr(X^\top Y)$, hence $K(X, Y) = -\tr(X^\top Y)$. The canonical Ad-invariant positive-definite metric is therefore
\[
\langle X, Y \rangle = -K(X, Y) = \tr(X^\top Y),
\]
which is the Frobenius form. The metric is determined up to a positive scalar, and the block weight $\lambda_\theta$ is a global scale only.
\end{remark}

\subsection{\texorpdfstring{$\SE(3)$}{SE(3)} --- Rigid motions of 3-space}
\label{sec:se3}

\textbf{Group.} $\SE(3) = \SO(3) \ltimes \R^3$, dimension 6:
\[
\SE(3) = \left\{ \begin{pmatrix} R & t \\ 0 & 1 \end{pmatrix} : R \in \SO(3),\; t \in \R^3 \right\}.
\]

\textbf{Algebra.}
\[
\algse(3) = \left\{ \begin{pmatrix} \Omega & v \\ 0 & 0 \end{pmatrix} : \Omega \in \algso(3),\; v \in \R^3 \right\} \cong \R^6.
\]
Bracket: $[(\Omega_1, v_1), (\Omega_2, v_2)] = ([\Omega_1, \Omega_2],\; \Omega_1 v_2 - \Omega_2 v_1)$.

\textbf{Exp/log.} Block structure identical to $\Aff(2)$ and $\SE(2)$:
\[
\exp(\Omega, v) = \bigl(e^\Omega,\; V(\Omega)\, v\bigr),
\]
with $e^\Omega$ via Rodrigues and
\[
V(\Omega) = I + \frac{1 - \cos\theta}{\theta^2}\, \Omega + \frac{\theta - \sin\theta}{\theta^3}\, \Omega^2,
\]
where $\theta$ is the rotation angle, defined coordinate-free by
\[
\theta \;=\; \sqrt{-\tfrac{1}{2}\,\tr(\Omega^2)} \;=\; \|\omega\|_2,
\]
with $\omega \in \R^3$ the axial vector satisfying $\Omega = [\omega]_\times$. Note that $\|\Omega\|_F = \theta\sqrt{2}$, so the Frobenius norm of $\Omega$ is \emph{not} the rotation angle; the factor $\sqrt 2$ corresponds to the orthonormal normalization of \cref{sec:block-decomposition}. Inverting on the chart $\theta \in (0, \pi)$:
\[
V(\Omega)^{-1} = I - \tfrac{1}{2} \Omega + \left( \frac{1}{\theta^2} - \frac{1 + \cos\theta}{2\theta \sin\theta} \right) \Omega^2.
\]

\textbf{Block decomposition.} Two blocks:
\[
\algse(3) = \underbrace{\R^3}_{\text{translation}} \oplus \underbrace{\R^3}_{\text{rotation}}.
\]

\textbf{Orthonormal basis.} Translation: $T_x, T_y, T_z$ (unit norm in the homogeneous embedding). Rotation: $R_\alpha = \frac{1}{\sqrt{2}} L_\alpha$.

\textbf{Score.} Two weights:
\[
s_{ij} = -\bigl(\lambda_t (\delta t_x^2 + \delta t_y^2 + \delta t_z^2) + \lambda_\theta (\delta\theta_x^2 + \delta\theta_y^2 + \delta\theta_z^2)\bigr) / \tau.
\]

Block-weight count per head: 2 (plus temperature). Same structure as $\SE(2)$, one dimension higher.

\begin{remark}[No Ad-invariant positive-definite metric on $\algse(3)$]
Same structure as $\SE(2)$. For $X = (0, v_X)$, $Y = (0, v_Y) \in \algt$ and $Z = (\Omega, 0) \in \algso(3)$: $[X, Y] = 0$ and $[X, Z] = (0, -\Omega v_X) \in \algt$. So $B|_{\algt \times \algt}(v_Y, \Omega v_X) = 0$. For fixed $v_X \neq 0$, $\{\Omega v_X : \Omega \in \algso(3)\} = v_X^\perp$. Every non-zero $w \in \R^3$ lies in $v_X^\perp$ for some $v_X$ (choose any $v_X \perp w$), so the union over $v_X$ is all of $\R^3$. Therefore $B|_{\algt \times \algt} = 0$, and no Ad-invariant positive-definite metric exists on $\algse(3)$.
\end{remark}

\subsection{\texorpdfstring{$\Aff(3)$}{Aff(3)} --- Spatial affine frames}
\label{sec:aff3}

The affine construction extends directly to spatial affine frames.  An $\Aff(3)$ token is
\[
\Aff(3) = \GL(3, \R) \ltimes \R^3 =
\left\{
\begin{pmatrix} A & t \\ 0 & 1 \end{pmatrix}
: A \in \GL(3, \R),\; t \in \R^3
\right\},
\]
with composition and inverse
\[
(A_1, t_1)(A_2, t_2) = (A_1 A_2,\; A_1 t_2 + t_1),
\qquad
(A, t)^{-1} = (A^{-1},\; -A^{-1} t).
\]
Its algebra is
\[
\algaff(3) = \alggl(3, \R) \ltimes \R^3,
\]
with bracket
\[
[(X_1, v_1), (X_2, v_2)] = ([X_1, X_2],\; X_1 v_2 - X_2 v_1).
\]
Under the natural orthogonal action, the algebra decomposes as
\[
\algaff(3) =
\underbrace{\R^3}_{\text{translation}}
\oplus \underbrace{\algso(3)}_{\text{rotation}}
\oplus \underbrace{\R}_{\text{iso.\ scale}}
\oplus \underbrace{\Sym_0(3)}_{\text{aniso.\ + shear}},
\]
with dimensions $3 + 3 + 1 + 5 = 12$. Thus $\Aff(3)$ has the same four geometric block types as $\Aff(2)$, but the traceless symmetric block grows from dimension $2$ to dimension $5$.

Writing
\[
w_{ij} = \log(g_i^{-1} g_j) = (u_{ij},\; \omega_{ij},\; \eta_{ij},\; q_{ij}) \in \R^3 \oplus \R^3 \oplus \R \oplus \R^5
\]
in orthonormal block coordinates, the block-weighted norm is
\[
\|w_{ij}\|_\lambda^2 = \lambda_t \|u_{ij}\|^2 + \lambda_\theta \|\omega_{ij}\|^2 + \lambda_s\, \eta_{ij}^2 + \lambda_q \|q_{ij}\|^2,
\]
and the score is $s_{ij} = -\|w_{ij}\|_\lambda^2 / \tau$. The number of score weights per head is therefore four, exactly as for $\Aff(2)$.

The exponential and logarithm use the same block-triangular form as in $\Aff(2)$:
\[
\exp(X, v) = \bigl(e^X,\; V(X)\, v\bigr),
\qquad
V(X) = \int_0^1 e^{r X}\, dr.
\]
On the principal chart
\[
\mathcal{U}_{\Aff(3)} = \{(A, t) : A \text{ has no eigenvalue on } \R_{\le 0}\},
\]
\[
\log(A, t) = \bigl(\log A,\; V(\log A)^{-1}\, t\bigr).
\]

\begin{remark}[No Ad-invariant positive-definite metric on $\algaff(3)$]
\label{rem:no-ad-invariant-aff3}
Let $D = (I_3, 0)$ be the isotropic-scale generator. Then $[D, (0, v)] = (0, v)$ for every translation $v \in \R^3$. Ad-invariance of a symmetric bilinear form $B$ gives, for $v_1, v_2 \in \algt$,
\[
B([D, v_1], v_2) + B(v_1, [D, v_2]) = 0 \;\Longrightarrow\; 2\, B(v_1, v_2) = 0,
\]
hence $B|_{\algt \times \algt} = 0$. For $X \in \alggl(3)$, $[D, X] = 0$, so $B([D, v], X) + B(v, [D, X]) = B(v, X) + 0 = 0$, giving $B(v, X) = 0$. The translation ideal lies in the radical of every Ad-invariant symmetric bilinear form; no Ad-invariant positive-definite metric exists on $\algaff(3)$. The construction uses the fixed positive-definite Frobenius form and its four-block refinement above.
\end{remark}

\subsection{Summary table}
\label{sec:instantiations:summary}

\begin{table}[ht]
\centering
\footnotesize
\renewcommand{\arraystretch}{1.2}
\setlength{\tabcolsep}{4pt}
\resizebox{\textwidth}{!}{%
\begin{tabular}{@{}l c l c l l@{}}
\toprule
Group & Dim & Blocks & Weights/head & Exp/log & Ad-inv.\ metric? \\
\midrule
$\SO(2)$ & 1 & rotation (1) & 1 & $R(\omega)$ / $\theta$ & yes (any $\lambda>0$; abelian, Killing $= 0$) \\
$\SE(2)$ & 3 & trans.\ (2) + rot.\ (1) & 2 & $V(\omega)$ closed-form & no (radical) \\
$\SO(3)$ & 3 & rotation (3) & 1 & Rodrigues & yes (Killing) \\
$\SE(3)$ & 6 & trans.\ (3) + rot.\ (3) & 2 & Rodrigues $+\ V(\Omega)$ & no (radical) \\
$\Aff(2)$ & 6 & trans.\ (2) + rot.\ (1) + scale (1) + shear (2) & 4 & Cayley--Hamilton & no (radical) \\
$\Aff(3)$ & 12 & trans.\ (3) + rot.\ (3) + scale (1) + aniso.+shear (5) & 4 & matrix exp/log $+\ V(X)$ & no (radical) \\
\bottomrule
\end{tabular}}
\caption{Six instantiations of the construction. For all groups, the score is $s_{ij} = -\|w_{ij}\|_\lambda^2 / \tau$ with $w_{ij} = \log(g_i^{-1} g_j)$.}
\label{tab:instantiations}
\end{table}

The per-group block count is not a design choice: it equals the number of $O(n)$-irreducible components of the algebra under the standard representation (Schur's lemma applied to $\algg$ as an $O(n)$-module). For $\Aff(2)$ the irreducibles are $V \oplus \mathbb{R} \oplus \mathbb{R} \oplus \mathrm{Sym}^2_0 V$ (translation $V = \R^2$, rotation $\Lambda^2 V \cong \R$, isotropic scale $\R$, anisotropic+shear $\mathrm{Sym}^2_0 V \cong \R^2$) --- four blocks, four weights. The group dictates $K$; we do not.

For compact simple algebras ($\algso(3)$), the Killing form provides a canonical metric and the block weight is a global scale only. For $\SO(2)$, the algebra is 1-dimensional and abelian; the Killing form is zero, but any positive scalar defines an Ad-invariant metric (the choice is absorbed by $\lambda_\theta$). For non-compact groups with a nonzero translation ideal ($\SE(2)$, $\SE(3)$, $\Aff(2)$), no Ad-invariant positive-definite metric exists. Here the block weights carry geometric content. They determine the relative importance of translation versus rotation, and of scale and shear for $\Aff(2)$.

\textbf{$\Aff(2)$ is the case no vector-token attention method reaches.} The irrep-and-Clebsch--Gordan tradition is restricted by unitarity to compact groups, and $\Aff(2)$ is non-compact non-abelian (no non-trivial finite-dimensional unitary irreps). LieTransformer~\cite{hutchinson2021lietransformer} attends over $(g, v)$ pairs and learns the kernel; the LieConv~\cite{finzi2020lieconv} lifting it builds on requires a surjective exponential, which $\Aff(2)$ fails. PONITA-style position-orientation reductions partially collapse the rotation block and have no place for scale or shear. Placing the token strictly on the group unlocks the full affine case, and with it the entire four-block decomposition above. That same decomposition extends verbatim to $\Aff(3)$ (\cref{sec:aff3}). The construction is dimension-uniform across $\Aff(n)$ for $n \ge 2$, via the same $O(n)$-irreducible split.

\section{Architecture}
\label{sec:architecture}

The construction of \cref{sec:construction} fits inside a standard transformer skeleton with one substitution. The backbone is a vanilla transformer with pre-layer normalization, multi-head attention, a feed-forward block, and residual connections. It operates directly on group-element tokens. The attention score is the algebra-norm score \eqref{eq:score} on the invariant $w_{ij}$, and the per-token output is the local-form correction \eqref{eq:output-local}. Each token is nothing but its group element; all geometry enters through $w_{ij}$. The backbone carries none of the representation-theoretic machinery enumerated in \cref{sec:intro} --- no irrep layers, no auxiliary frames, no equivariant message-passing primitives. The equivariance lives in the token type and the score; the network adds nothing geometric on top.

\subsection{Set-input transformer}
\label{sec:set-input}

The input is a permutation-invariant set $\{g_i\}_{i=1}^N \subset G$ of $N$ group-valued tokens. All tokens are initialized with the \emph{same} learned vector,
\[
h_i^{(0)} = h_0 \in \R^d, \qquad i = 1, \ldots, N,
\]
so that all token-distinguishing information must enter through the geometric relations $w_{ij}$. No positional encoding is used; spatial position lives in the group element $g_i$ itself and is accessed through $w_{ij}$. The distinction matters for the claim that tokens carry no payload. The \emph{token} is the bare group element $g_i$. The vector $h_i$ is a \emph{hidden state}: identical across tokens at initialization, it is filled in only from the relations $w_{ij}$ during the forward pass, never a feature attached to a particular token.

Diagonal entries of the attention map are masked, $s_{ii} = -\infty$, preventing trivial self-attention from $w_{ii} = 0$. In the experiments of \cref{sec:experiments} the architecture uses $L = 3$ layers, embedding dimension $d = 32$, and $H = 4$ attention heads, on $7$-token model inputs (one element held out of each length-$8$ sequence).

\subsection{Forward pass}
\label{sec:forward}

The pairwise invariant is computed once at the input and reused across all $L$ layers:
\[
w_{ij} = \log\bigl(g_i^{-1} g_j\bigr) \in \R^{\dim \algg}, \qquad (i, j) \in \{1, \ldots, N\}^2.
\]
Each transformer layer is a pre-LN block:
\begin{align*}
h^{(\ell + \tfrac{1}{2})} &= h^{(\ell)} + \mathrm{Attn}\bigl(\mathrm{LN}(h^{(\ell)}),\, w\bigr), \\
h^{(\ell + 1)} &= h^{(\ell + \tfrac{1}{2})} + \mathrm{FFN}\bigl(\mathrm{LN}(h^{(\ell + \tfrac{1}{2})})\bigr),
\end{align*}
where the FFN is two linear layers with hidden width $4d$ and a GELU non-linearity. The attention module $\mathrm{Attn}(\cdot, w)$ produces per-token updates as follows.

\textbf{Scores.} For each head $k = 1, \ldots, H$:
\[
s_{ij}^{(k)} = -\|w_{ij}\|_{\lambda_k}^2 \,/\, \tau_k, \qquad \alpha_{ij}^{(k)} = \mathrm{softmax}_j\bigl(s_{ij}^{(k)}\bigr),
\]
with the diagonal masked $s_{ii}^{(k)} = -\infty$ before softmax. The block-weighted norm $\|\cdot\|_{\lambda_k}^2$ is the per-head instance of \eqref{eq:weighted-norm}; the number of weights per head equals the block count of Table~\ref{tab:instantiations}.

\textbf{Values.} Each pair $(i, j)$ contributes a value vector
\[
V_{ij} = W_V\bigl[h_j;\, w_{ij}\bigr] \in \R^d,
\]
formed by concatenating the source hidden state $h_j$ with the relative pose $w_{ij}$ and projecting through a single linear map $W_V \in \R^{d \times (d + \dim \algg)}$. The result is reshaped into $H$ heads of dimension $d/H$ and combined with the attention weights $\alpha^{(k)}$ to produce the layer output. An output linear $W_O$ then projects this output back to $\R^d$.

We use $V_{ij} = W_V[h_j; w_{ij}]$ instead of the standard $V_{ij} = W_V h_j$. This makes the value pathway carry the geometric difference vector, not just the source hidden state. This preserves directional information that the score loses when it takes a squared norm: $\|w_{ij}\|_\lambda^2 = \|-w_{ij}\|_\lambda^2$, but the values $V_{ij}$ and $V_{ji}$ are distinct.

\subsection{Multi-head and block-weight parameterization}
\label{sec:multihead}

Each head $k$ owns its own block-weight vector $\lambda_k \in \R_{>0}^K$ and temperature $\tau_k > 0$, where $K$ is the algebra's block count (Table~\ref{tab:instantiations}). Positivity is enforced by a softplus reparameterization:
\[
\lambda_k = \mathrm{softplus}(\tilde{\lambda}_k) + \varepsilon, \qquad \tau_k = \mathrm{softplus}(\tilde{\tau}_k) + \varepsilon,
\]
with $\tilde{\lambda}_k, \tilde{\tau}_k \in \R$ the unconstrained learned parameters and $\varepsilon = 10^{-3}$ a small floor. The unconstrained parameters are initialized at zero, giving $\lambda_k \approx \tau_k \approx \ln 2 \approx 0.69$ at initialization --- isotropic across blocks, on the order of unity.

The total number of score parameters in $L$ layers with $H$ heads is $L \cdot H \cdot (K + 1)$. For the $\SE(2)$ experiments ($K = 2$, $H = 4$, $L = 3$), this is $36$.

\subsection{Output head}
\label{sec:output-head}

The per-token correction is produced by a two-layer MLP applied to the final hidden state $h_i^{(L)}$:
\[
\delta_i \;=\; \phi_\delta\bigl(h_i^{(L)}\bigr) \;=\; W_2\, \mathrm{GELU}\bigl(W_1 h_i^{(L)} + b_1\bigr) + b_2 \;\in\; \R^{\dim \algg},
\]
with $W_1 \in \R^{d \times d}$, $W_2 \in \R^{\dim \algg \times d}$. The output pose is the local form of Theorem~\ref{thm:equivariant-output}:
\[
\hat{g}_i = g_i\, \exp(\delta_i).
\]
Since $\delta_i$ is a function of the invariant features $h_i^{(L)}$, it is itself invariant under the diagonal $G$ action. The composition $\hat{g}_i$ is therefore equivariant. Task-specific output heads --- for example, a per-token gap-detection head in the sequence-completion experiments of \cref{sec:experiments} --- are added on top of this base in the same invariant-features-in, scalar-out pattern.

\subsection{Equivariance and parameter count}
\label{sec:arch-params}

Equivariance is structural rather than learned. The input $w_{ij}$ is invariant by Lemma~\ref{lem:invariance}. Every subsequent operation (score, softmax, value, FFN, layer norm) is a function of invariant features and produces invariant features. The output head produces invariant $\delta_i$; the final composition $\hat{g}_i = g_i \exp(\delta_i)$ is equivariant by Theorem~\ref{thm:equivariant-output}. No training signal enforces equivariance and no equivariance-error term appears in the loss.

For the $\SE(2)$ instance (\cref{sec:exp:se2}), the parameter counts are (per model, total parameters / score-only parameters); the per-group score-parameter counts for all three experiments are given in \cref{sec:experiments}:
\begin{center}
\renewcommand{\arraystretch}{1.15}
\begin{tabular}{@{}lrr@{}}
\toprule
Model & Total & Score-only \\
\midrule
G --- algebra-norm score \eqref{eq:score} & $\sim$33\,000 & 36 \\
C --- learned MLP kernel on the same $w_{ij}$ & $\sim$35\,000 & 1\,932 \\
A --- vanilla scaled-dot-product on absolute coordinates & $\sim$40\,000 & ---  \\
\bottomrule
\end{tabular}
\end{center}
The score-parameter ratio between G and C is $\sim$54$\times$. Model C uses the same invariant $w_{ij}$ as G but replaces the closed-form score \eqref{eq:score} with a per-head MLP $\psi: \R^{\dim \algg} \to \R$; everything else is identical. Model A uses the standard scaled-dot-product score on absolute coordinate features (e.g., $(\cos \theta_i, \sin \theta_i, t_{x,i}, t_{y,i})$ for $\SE(2)$) and serves as the equivariance-breaking control. The architecture is otherwise identical across G, C, A: same depth, width, head count, FFN, and output head.

\section{Experiments}
\label{sec:experiments}

We validate the construction on three groups spanning every block-count level of \cref{tab:instantiations}: $\SE(2)$ (two blocks), $\SO(3)$ (one block), and $\Aff(2)$ (four blocks). The task in each case is \emph{sequence completion}: a constant-step sequence of $N{=}8$ group elements is generated, one interior token is removed and the remaining seven are permuted; the model must reconstruct the missing pose from the unordered set. Three models are compared in each experiment:

\begin{itemize}[leftmargin=*,topsep=2pt,itemsep=0pt]
\item \textbf{G} --- algebra-norm score $s_{ij} = -\|w_{ij}\|_\lambda^2 / \tau$ from \eqref{eq:score}.
\item \textbf{C} --- same invariant $w_{ij}$, score replaced by a per-head MLP $\psi: \R^{\dim \algg} \to \R$ (one hidden layer of $32$ units, ReLU). This is the LieTransformer-style move of \emph{learning} the kernel on the relative-pose invariant; against G it isolates fixed-versus-learned scoring, with the invariant and the rest of the architecture held constant.
\item \textbf{A} --- the \emph{vector-token ontology held fixed}: tokens are flat absolute features in $\R^d$ (specified per experiment) and the score is vanilla scaled-dot-product attention. This is the standard equivariant-ML schema (token = vector, no intrinsic invariant) operating on the same data, and serves as the direct ontology-comparison control. It isolates the group-token-versus-vector-token ontology; the finer distinction between a bare group element and a $(g,v)$ pair (the contrast with LieTransformer) is argued in \cref{sec:related}, not ablated here.
\end{itemize}

All three architectures are otherwise identical: the transformer of \cref{sec:architecture} with $L{=}3$ layers, $H{=}4$ heads, $d{=}32$, the gap-detection head, the local-form correction head, and the same value pathway $V_{ij} = W_V[h_j; w_{ij}]$ on G and C. Training: Adam at $10^{-3}$, batch $64$, gradient clip $2.0$, $200$ epochs, $5{,}000$ training instances, $500$ each for validation and test. Each configuration is run with three random seeds; we report mean $\pm$ std across seeds.

\textbf{Metrics (common).} Pose error $E_g = \|\log(\hat{g}^{-1} g_j)\|^2$ on the test set in physical coordinates (Frobenius fallback near the chart boundary; in all three experiments the fallback was triggered by no test instance). Flanking accuracy: indicator that the predicted base-token index $\hat{i} = \arg\max_i p_i$ is one of the two true neighbors $\{j-1, j+1\}$. Equivariance error: apply $10$ random global $a \in G$ per instance and measure $\|\log((a \hat{g})^{-1} \hat{g}(a \cdot \mathcal{S}))\|^2$, averaged.

\subsection{\texorpdfstring{$\SE(2)$}{SE(2)} sequence completion}
\label{sec:exp:se2}

\textbf{Task.} Generate $g_k = g_0 \cdot h^k$ on $\SE(2)$ with $g_0$ sampled across the plane and $h = \exp(\omega_h, v_h)$ small enough that $|\omega_h| < \pi/8$ keeps every relative pose on the principal chart.

\textbf{Score parameters.} G uses $2$ block weights $+$ temperature per head $= 3$ score parameters per head, for $L \cdot H \cdot 3 = 36$ total. C's per-head MLP $\R^3 \to \R$ has $\sim$$161$ score parameters per head, $1{,}932$ total. \textbf{C/G ratio: $54\times$.}

\textbf{Absolute features for A.} $v_i = (\cos\theta_i, \sin\theta_i, t_{x,i}, t_{y,i}) \in \R^4$.

\textbf{Results} (Table~\ref{tab:exp:se2}, Figure~\ref{fig:exp:se2}).
\begin{table}[ht]
\centering
\small
\renewcommand{\arraystretch}{1.15}
\setlength{\tabcolsep}{6pt}
\begin{tabular}{@{}lccc@{}}
\toprule
Model & Pose error $\downarrow$ & Flanking acc.\ $\uparrow$ & Equivariance error $\downarrow$ \\
\midrule
G & $0.003 \pm 0.001$ & $1.000 \pm 0.000$ & $4.0 \times 10^{-10}$ \\
C & $0.005 \pm 0.001$ & $1.000 \pm 0.000$ & $1.3 \times 10^{-12}$ \\
A & $0.069 \pm 0.002$ & $0.999 \pm 0.001$ & $1.3 \times 10^{-1}$ \\
\bottomrule
\end{tabular}
\caption{Sequence completion on $\SE(2)$. Mean $\pm$ std over 3 seeds on 500 test instances.}
\label{tab:exp:se2}
\end{table}

\begin{figure}[ht]
\centering
\includegraphics[width=\textwidth]{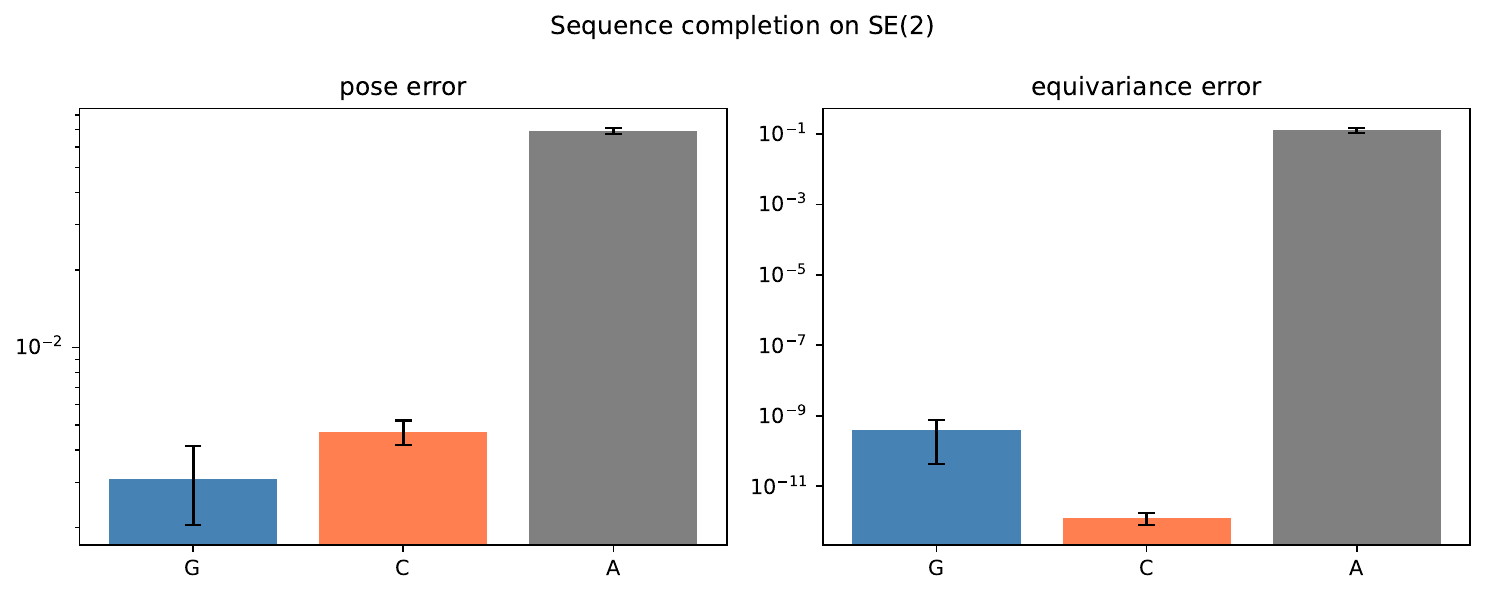}
\caption{$\SE(2)$ sequence completion: pose error and equivariance error per model (log scale). G is $34\%$ \emph{better} than C at $54\times$ fewer score parameters; both vastly outperform the equivariance-blind A, whose equivariance error is eleven orders of magnitude larger.}
\label{fig:exp:se2}
\end{figure}

G uses $36$ score parameters and reaches pose error $0.003$; C uses $1{,}932$ ($54\times$ more) and reaches $0.005$. The closed-form algebra-norm score is $34\%$ \emph{better} than the learned MLP on the same invariant, with $54\times$ fewer parameters. G's and C's equivariance errors are at the float32 numerical floor; A's is $\sim 10^{-1}$, confirming that the task tests invariance and that G/C are equivariant by construction.

\subsection{\texorpdfstring{$\SO(3)$}{SO(3)} sequence completion}
\label{sec:exp:so3}

\textbf{Task.} Generate $g_k = g_0 \cdot h^k$ on $\SO(3)$ with $g_0$ sampled Haar-uniformly and $h = \exp([\omega_h]_\times)$ with $\|\omega_h\| \le \pi/8$, so that powers $h^k$ for $k \le 7$ stay on the principal chart. Same sequence-completion task: remove one interior token, reconstruct.

\textbf{Score parameters.} $\algso(3)$ has one block (the algebra is simple), so G uses one block weight $\lambda_\theta$ plus temperature per head $= 2$ score parameters per head, $L \cdot H \cdot 2 = 24$ total. C's per-head MLP $\R^3 \to \R$ has $\sim$$161$ score parameters per head, $1{,}932$ total. \textbf{C/G ratio: $80\times$.}

\textbf{Absolute features for A.} $v_i = \mathrm{vec}(R_i) \in \R^9$ (flattened rotation matrix).

\textbf{Results} (Table~\ref{tab:exp:so3}, Figure~\ref{fig:exp:so3}).
\begin{table}[ht]
\centering
\small
\renewcommand{\arraystretch}{1.15}
\setlength{\tabcolsep}{6pt}
\begin{tabular}{@{}lccc@{}}
\toprule
Model & Pose error $\downarrow$ & Flanking acc.\ $\uparrow$ & Equivariance error $\downarrow$ \\
\midrule
G & $(1.8 \pm 0.1) \times 10^{-4}$ & $0.998 \pm 0.000$ & $1.6 \times 10^{-14}$ \\
C & $(1.3 \pm 0.5) \times 10^{-4}$ & $0.998 \pm 0.001$ & $2.7 \times 10^{-14}$ \\
A & $0.067 \pm 0.002$ & $0.965 \pm 0.006$ & $7.4 \times 10^{-2}$ \\
\bottomrule
\end{tabular}
\caption{Sequence completion on $\SO(3)$. Mean $\pm$ std over 3 seeds on 500 test instances.}
\label{tab:exp:so3}
\end{table}

\begin{figure}[ht]
\centering
\includegraphics[width=\textwidth]{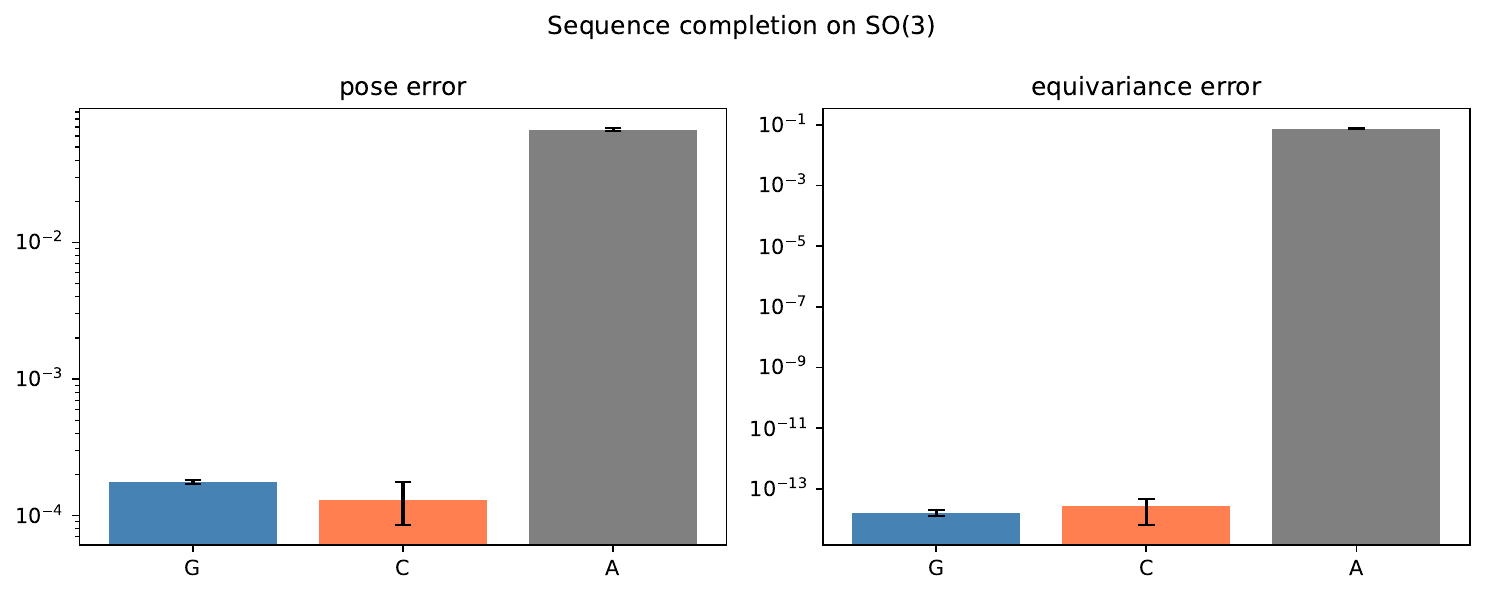}
\caption{$\SO(3)$ sequence completion: pose error and equivariance error per model (log scale). G with 24 score parameters matches C at $80\times$ fewer parameters (both at the $10^{-4}$ floor); G's and C's equivariance errors sit at $\sim 10^{-14}$ in float32, while A's is $\sim 10^{-1}$.}
\label{fig:exp:so3}
\end{figure}

Both G and C reach pose error $\sim$$10^{-4}$ (statistically indistinguishable within seed noise); C's third seed lands lower and accounts for the wider C-std, but the mean difference is well inside the spread of either model. Equivariance errors for G and C sit at $\sim 10^{-14}$, effectively float64 precision in float32 arithmetic. These are the cleanest of the three groups, because the $\SO(3)$ Rodrigues path is a single closed-form branch with no eigenvalue case-splits. A fails by twelve orders of magnitude on equivariance and by $380\times$ on pose; its flanking accuracy also drops to $0.965$, the only configuration where the gap head is not effectively saturated.

\subsection{\texorpdfstring{$\Aff(2)$}{Aff(2)} sequence completion}
\label{sec:exp:aff2}

\textbf{Task.} Generate $g_k = g_0 \cdot h^k$ on $\Aff(2)$ with $g_0 = (A_0, t_0)$ sampled generously across the group (rotation, anisotropic scale and shear of bounded magnitude, translation $\sim \mathcal{N}(0, 9 I_2)$). The step $h = \exp(X_h, v_h)$ has its linear-part magnitude bounded so that $A_h^k$ for $k \le 7$ has no eigenvalue on $\R_{\le 0}$ (chart safety on the principal $\log$ chart of $\GL(2,\R)$).

\textbf{Score parameters.} $\algaff(2)$ has four blocks (\eqref{eq:aff2-blocks}), so G uses $4$ block weights $+$ temperature per head $= 5$ score parameters per head, $L \cdot H \cdot 5 = 60$ total. C's per-head MLP $\R^6 \to \R$ has $\sim$$257$ score parameters per head, $3{,}084$ total. \textbf{C/G ratio: $51\times$.}

\textbf{Absolute features for A.} $v_i = (A_{i,11}, A_{i,12}, A_{i,21}, A_{i,22}, t_{x,i}, t_{y,i}) \in \R^6$.

\textbf{Results} (Table~\ref{tab:exp:aff2}, Figure~\ref{fig:exp:aff2}).
\begin{table}[ht]
\centering
\small
\renewcommand{\arraystretch}{1.15}
\setlength{\tabcolsep}{6pt}
\begin{tabular}{@{}lccc@{}}
\toprule
Model & Pose error $\downarrow$ & Flanking acc.\ $\uparrow$ & Equivariance error $\downarrow$ \\
\midrule
G & $0.007 \pm 0.002$ & $1.000 \pm 0.000$ & $2.3 \times 10^{-5}$ \\
C & $0.007 \pm 0.001$ & $1.000 \pm 0.000$ & $1.4 \times 10^{-9}$ \\
A & $0.79 \pm 0.10$  & $0.953 \pm 0.010$ & $1.29$ \\
\bottomrule
\end{tabular}
\caption{Sequence completion on $\Aff(2)$. Mean $\pm$ std over 3 seeds on 500 test instances.}
\label{tab:exp:aff2}
\end{table}

\begin{figure}[ht]
\centering
\includegraphics[width=\textwidth]{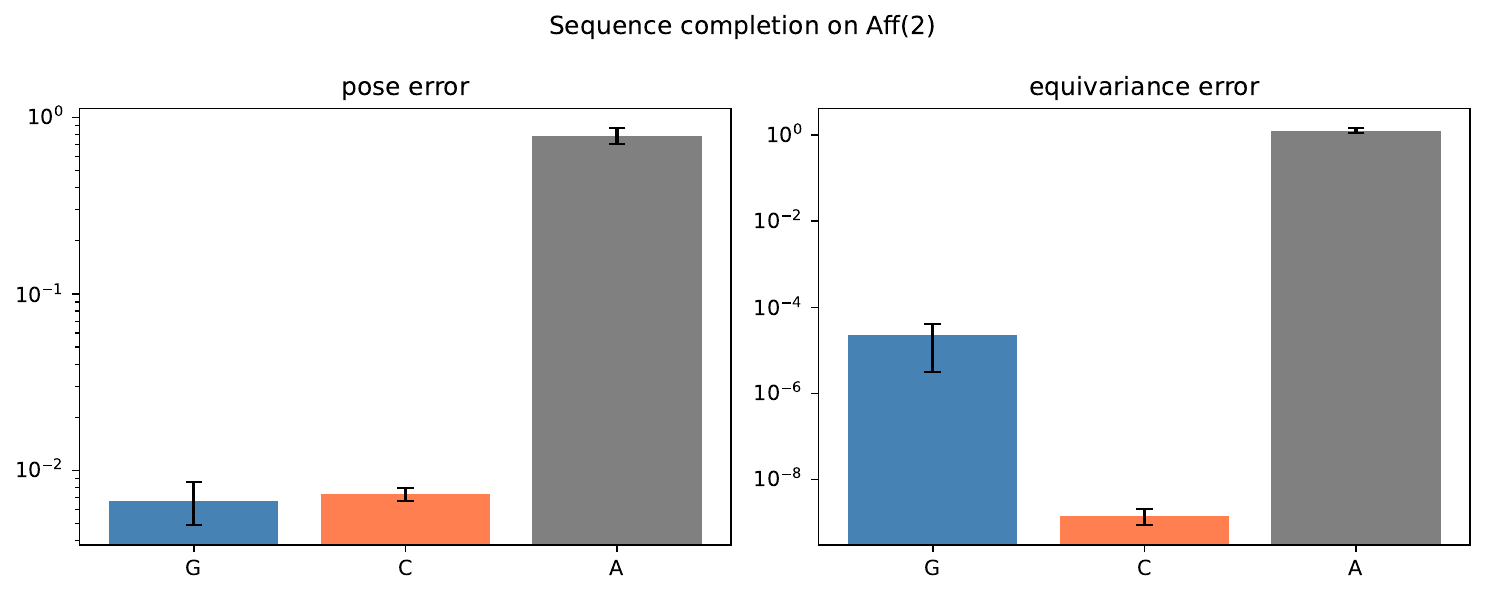}
\caption{$\Aff(2)$ sequence completion: pose error and equivariance error per model (log scale). G with 60 score parameters matches C at $51\times$ fewer parameters; A's pose error is $117\times$ worse than G's, and its equivariance error is $1.29$ --- five orders of magnitude above G's and nine above C's.}
\label{fig:exp:aff2}
\end{figure}

G and C reach pose error $\approx 0.007$, indistinguishable within seed noise, despite C using $51\times$ more score parameters. A collapses. Its pose error is $0.79$ ($117\times$ G's). Its equivariance error is $1.29$, well above the inputs themselves, so the model produces an essentially independent output under global transforms. The gap-detection head also degrades, to $0.95$ flanking accuracy. This is the only group where lack of structural invariance also breaks token-identification performance. G's equivariance error of $2.3\times 10^{-5}$ is several orders of magnitude above the $\SO(3)$ and $\SE(2)$ values. This is a float32 accumulation artifact of the branched Cayley--Hamilton path for the $\GL(2)$ matrix exp and log, with its real-distinct, real-repeated, complex-conjugate, and near-zero eigenvalue cases. It is not a structural failure. The construction is exactly equivariant in arithmetic for $\Aff(2)$ as for the other groups.

\subsection{Interpretation}
\label{sec:exp:interpretation}

The three experiments deliver three claims about the construction, backed by every block-count level of the instantiation table.

\textbf{The closed-form score is canonical, not merely compact.} Model C gives a learned MLP kernel one to two orders of magnitude more score parameters on the \emph{same} invariant $w_{ij}$, $50$ to $80\times$ G's count ($24$ vs $1{,}932$ on $\SO(3)$, $36$ vs $1{,}932$ on $\SE(2)$, $60$ vs $3{,}084$ on $\Aff(2)$). It does not improve on it. G \emph{outperforms} C on $\SE(2)$ by $\sim$$34\%$ and \emph{matches} it within seed noise on $\SO(3)$ and $\Aff(2)$, with the $\SO(3)$ point estimate marginally favouring C. The MLP has every bit of capacity needed to represent G's kernel and more, yet cannot reliably turn the surplus into a better attention pattern. On these tasks a far higher-capacity learned kernel does not beat the closed-form norm --- consistent with the block-weighted algebra norm being the canonical readout of the invariant rather than a compression of some better learned score. The $50$ to $80\times$ parameter saving is a corollary of that, not the headline. These tasks are deliberately simple --- flanking accuracy saturates at $1.000$ on $\SE(2)$ and $\Aff(2)$ --- so the parity bounds what a learned kernel buys in this regime rather than settling it in general; a harder, non-saturated task is left to follow-up.

\textbf{Structural equivariance.} G's and C's equivariance errors stay at the float32 numerical floor across all three groups: $\sim 10^{-14}$ on $\SO(3)$, $\sim 10^{-10}$--$10^{-12}$ on $\SE(2)$, $\sim 10^{-5}$--$10^{-9}$ on $\Aff(2)$. The differences across groups track the length of the matrix exp/log path: single-branch Rodrigues for $\SO(3)$, two-block triangular for $\SE(2)$, four-branch Cayley--Hamilton for $\Aff(2)$. They do not reflect any structural property. The construction is exactly equivariant in arithmetic for every matrix Lie group it covers. A breaks equivariance by five (on $\Aff(2)$, relative to G's higher floor) to twelve (on $\SO(3)$, relative to G's $10^{-14}$) orders of magnitude --- confirming that the tasks test invariance and that G's equivariance is structural, not trained.

\textbf{Reaching the affine regime.} The three experiments cover every block-count level of the construction's instantiation: $K = 1$ ($\SO(3)$), $K = 2$ ($\SE(2)$), $K = 4$ ($\Aff(2)$); the closed-form algebra-norm score works at every level with no architectural change, only the number of block weights ($K+1$ per head including temperature). The $\Aff(2)$ experiment is the load-bearing one: its four-block decomposition --- translation, rotation, isotropic scale, anisotropic-and-shear --- realizes in a transformer the non-compact non-abelian affine regime that no irrep or surjective-exp attention method reaches, the first such demonstration we are aware of. On the constant-step data the per-instance invariants are collinear, so the four block weights enter only through a single scalar and are not separately identified. The experiment shows the construction runs and is equivariant in this regime, not that the block weighting itself is exercised. A task separating the blocks is left to the same follow-up.

\section{Related Work}
\label{sec:related}

Every prior tradition keeps the token a vector and lets the group act on it externally; the structural limitation each inherits is a symptom of that shared ontology. We compare to nine traditions below, organized by token type, group operations, and the limitation each ontology imposes. The construction here uses none of these enforcement mechanisms, because placing the token on $G$ makes equivariance tautological rather than something to enforce.

\begin{table}[ht]
\centering
\footnotesize
\renewcommand{\arraystretch}{1.2}
\setlength{\tabcolsep}{4pt}
\resizebox{\textwidth}{!}{%
\begin{tabular}{@{}l l l l@{}}
\toprule
Tradition & Token domain & Group ops & Structural limitation \\
\midrule
Irrep / harmonic & vector in irrep space & representation & no non-compact irreps \\
EGNN & $(x_i,h_i) \in \R^n \times \R^d$ & distances + coordinate updates & point tokens; no frames/poses \\
Geometric algebra & multivector in flat $\R^d$ & sandwich product & token off the manifold \\
Capsules & unconstrained $4 \times 4$ matrix & heuristic & no Lie structure \\
Frame-augmented (IPA) & vector + frame $T_i \in \SE(3)$ & hybrid & position-only kernel \\
LieTransformer & $(g, v) \in G \times \R^d$ & learned kernel on pairs & surjective-exp lift (LieConv) \\
PONITA (position-orient.) & $(x, n) \in \R^n \times S^{n-1}$ & homogeneous-space & rotation block reduced; $\SE(n)$ only \\
Mironenco--Forr\'e (Lie decomp.) & function on $\Aff^+(n)$ & Cartan/polar, Haar-MC & CNN; not transformer \\
RoPE / RiemannFormer & vector in $\R^d$ & operational & not manifold-valued \\
\midrule
\textbf{This work} & $\mathbf{g_i \in G}$ on the group & closed-form algebra norm & --- \\
\bottomrule
\end{tabular}}
\caption{Attention methods by token type: every prior token is a vector carrying an external group action; this work alone places the token on the group. Listed by token domain, group operations, and structural limitation; sources cited in subsections.}
\label{tab:related}
\end{table}

\subsection{Irrep / harmonic analysis}
TFN~\cite{thomas2018tfn}, $\SE(3)$-Transformer~\cite{fuchs2020se3}, Equiformer~\cite{liao2022equiformer}, MACE~\cite{batatia2022mace}, group-equivariant CNNs~\cite{cohen2016group}, and steerable CNNs~\cite{weiler2019general} act on tokens that live in vector spaces carrying group representations. Equivariance is achieved via Clebsch--Gordan tensor products of irreps, restricted by unitarity to compact groups. The translation part of $\SE(n)$ is handled via relative positions (an abelian normal subgroup factoring trick); the rotation part uses irreps of $\SO(n)$. Non-trivial unitary irreps of non-compact non-abelian groups are infinite-dimensional, so the standard recipe does not extend to $\Aff(n)$. The tokens are vectors with $\rho(g)$ acting on them, not points on the group manifold. The line continues to the current state of the art --- EquiformerV3~\cite{equiformerv3_2026} is a 2026 irrep $\SE(3)$ graph-attention transformer.

\subsection{\texorpdfstring{$E(n)$}{E(n)}-equivariant graph networks}
EGNN~\cite{satorras2021egnn} is a central point-token baseline for molecules and point clouds. Its token is $(x_i,h_i)$: a point $x_i \in \R^n$ plus invariant features $h_i \in \R^d$.\footnote{EGNN's $h_i$ is a feature \emph{attached} to the token. The hidden state $h_i$ of \cref{sec:architecture} reuses the symbol but is not a token payload: it is identical across tokens at initialization and acquires content only from the relations $w_{ij}$ (\cref{sec:set-input}).} Messages depend on invariant squared distances $\|x_i-x_j\|^2$, and coordinate updates are built from relative displacement vectors $x_i-x_j$, giving $E(n)$ equivariance without irreps or spherical harmonics. This is close in spirit to using native geometric quantities, but the token is still a point with attached features, not a frame or transformation. There is no group-valued token $g_i$, no relative pose $g_i^{-1}g_j$, and no unified rotation/translation/shear logarithm. EGNN is therefore the natural Euclidean point baseline, whereas the construction here is the corresponding frame/pose-token construction.

\subsection{Geometric algebra}
GATr~\cite{brehmer2023gatr}, DriveGATr~\cite{drivegatr2026}, and the algebra-choice paper~\cite{dehaan2024algebra} represent tokens as multivectors in a fixed-dimension Clifford algebra. The group acts via the sandwich product. Tokens live in flat ambient space, \emph{not on the group}: motors representing $\SE(3)$ are a 6-dimensional submanifold of the 8-dimensional even subalgebra of $\mathrm{Cl}(3,0,1)$, and tokens are unconstrained 16-dimensional multivectors. The geometry is in the algebra structure, not the token space.

\subsection{Capsule networks}
Matrix Capsules with EM routing~\cite{hinton2018matrix} represent each entity by a $4 \times 4$ pose matrix updated by voting and EM routing. The pose matrices are unconstrained reals, not Lie group elements; the voting protocol provides soft training pressure toward affine-like composition, but no architectural constraint enforces a group structure. No Lie-algebra machinery, no exp/log, no equivariance theorem. Spatial Transformer Networks~\cite{jaderberg2015spatial} use $\Aff(2)$ as a learnable warping parameter, not a token representation.

\subsection{Frame-augmented (IPA)}
AlphaFold's Invariant Point Attention~\cite{jumper2021alphafold} attaches a frame $T_i \in \SE(3)$ to each vector token $s_i$. The attention score uses a squared-distance kernel between points expressed in pairs of local frames, which is $\SE(3)$-invariant. This is the closest published precedent for the construction: a squared-distance kernel between geometric quantities. The limitation is that only the position part is scored this way. Rotation is handled separately via frame transformation, and the frame is auxiliary. The primary token remains the vector $s_i$. The construction unifies rotation and translation in a single quantity, $w_{ij} = \log(g_i^{-1} g_j)$, and places the token on the group rather than in $\R^d$. A recent instance of the same pattern, ActionFlow~\cite{actionflow_2024}, carries $\SE(3)$ poses as metadata on feature tokens for robot policies, with a point-distance score.

\subsection{LieTransformer}
LieTransformer~\cite{hutchinson2021lietransformer} performs self-attention over pairs of elements in $G \times \R^d$ for an arbitrary Lie group, building on the LieConv~\cite{finzi2020lieconv} group lifting. Attention is computed on group-element pairs --- the closest existing recipe to native group-attention. Two differences set it apart from the construction here. The features are $(g, v)$ pairs, not bare group elements. And the kernel is a learned MLP on the relative-pose invariant, not the closed-form algebra norm; this is the fixed-versus-learned axis that Model C isolates in \cref{sec:experiments}. The affine groups lie outside its reach for a separate, inherited reason: the LieConv lifting requires a surjective exponential, which $\Aff(n)$ lacks. The construction places the token strictly on $G$ and applies on the principal-log chart of any matrix Lie group. The relative-pose log-map invariant itself appears earlier still, as the convolution-kernel coordinate of LieConv~\cite{finzi2020lieconv}. The cancellation-based equivariance argument is likewise anticipated by LieTransformer (the closed-form $\exp/\log$ on $\SO(n)/\SE(n)$ is classical); what is new here is the bare group-element token, the closed-form algebra-norm score, and the affine reach beyond surjective-exp groups.

\subsection{Position-orientation networks (PONITA)}
PONITA~\cite{bekkers2024ponita} is the closest published work to the $\SE(2)$ and $\SE(3)$ cases of the construction. It operates on the homogeneous space $\SE(n) / \SO(n - 1) \cong \R^n \times S^{n - 1}$ --- positions equipped with a single orientation vector --- and uses message passing with attention weight-shared over relative position-orientation pairs, derived from homogeneous-space theory. Provably $\SE(n)$-equivariant by construction. Two differences. First, the token is a position-with-single-orientation pair, not a full frame on the group; the rotation block of $\algse(n)$ is partially reduced. Second, the invariants used for weight-sharing are geometric quantities chosen per task (relative distance, relative orientation angles, alignments) rather than read off a single algebra log; the unified algebra-log form $w_{ij} = \log(g_i^{-1} g_j) \in \algse(n)$ used here recovers these as special cases, see \cref{sec:so2} (RoPE-class) and the IPA discussion above. The construction also extends naturally to $\Aff(2)$ where the position-orientation reduction is unavailable.

\subsection{Lie-group decompositions for non-compact equivariance}
Mironenco and Forr\'e~\cite{mironenco2024lie} target non-compact non-abelian groups including $G^+ = \mathrm{GL}^+(n,\R)$ and the orientation-preserving affine groups $\R^n \rtimes G^+$, which contain $\Aff^+(2)$. The exponential map for $\Aff(n)$ is not surjective, so the standard ``lift to algebra, exponentiate'' recipe fails. The paper's contribution is to use the Cartan/polar decomposition to factor $G^+$ globally as $\mathrm{Pos}(n) \times \SO(n)$, a symmetric-positive-definite part times a compact rotation part, on which Haar measure factors. The architecture is a group-equivariant CNN over images, not a transformer over framed primitives. The construction in this paper takes the same group as a target but uses a different mechanism --- attention over invariants on the principal-log chart --- and produces a different kind of model. Mironenco--Forr\'e give a rigorous measure-theoretic treatment of the $\Aff(n)$ symmetry via the Cartan/polar decomposition; the construction here is the first to embed $\Aff(2)$ tokens in a transformer architecture. Reductive Lie Neurons~\cite{kim2025relns} reach $\GL(n)$ with Lie-algebra-valued features in a general (non-attention) network, extending the non-compact reach beyond Mironenco--Forr\'e's convolution; neither is an attention construction on group-element tokens, so the affine case scored by a closed-form invariant attention norm remains open.

\subsection{RoPE and geometric reframings}
RoPE~\cite{su2021rope} uses $\SO(2)$ rotations applied to pairs of feature dimensions for positional encoding. The construction at \cref{sec:so2} subsumes the kernel-form question by giving the score $-\lambda \delta\theta^2 / \tau$. This depends on the same invariant, but with quadratic rather than oscillatory dependence. LieRE~\cite{ostmeier2024liere} generalizes RoPE to higher-dimensional Lie groups by learning a basis of skew-symmetric matrices $A_i$ and applying the rotation $R(p) = \exp(\sum_i p_i A_i)$ to keys and queries; the token remains a feature vector, the Lie rotation is positional encoding acting on it. RiemannFormer~\cite{ji2025riemannformer} imports geometric vocabulary (tangent vectors, parallel transport, Lie-algebra positional encoding) while keeping tokens as vectors in $\R^d$; the paper explicitly states it avoids the continuous geometry. All three keep the vector-token ontology --- the construction differs by actually placing the token on the manifold.

\subsection{Concurrent and recent work (2024--2026)}
Several 2024--2026 systems approach the same territory from the vector-token side and stop at the schema boundary; we group them by recency rather than tradition, as each extends a family already covered above. The Geometry-Aware Attention mechanism (GTA)~\cite{miyato2024gta} applies group representations to queries, keys, and values to align coordinate frames between multi-view tokens. The tokens remain vectors, and the group acts on them by $\rho(g)$. GTA already uses the relative pose, transforming keys and values by $\rho(g_i g_j^{-1})$, but it does so as an external representation action on a vector token, with no logarithm and no norm. The Clebsch--Gordan Transformer~\cite{cobb2025cgtransformer} builds global equivariant attention through irreducible representations and Clebsch--Gordan tensor products --- the irrep tradition in its purest transformer form, restricted to compact groups by unitarity. The Vanilla Group Equivariant Vision Transformer~\cite{fu2026gevit} renders ViT modules (patch embedding, self-attention, positional encoding, sampling) equivariant through enforced inductive biases; tokens are feature vectors throughout. None of these crosses to token-as-group-element; in each, the limit we highlight tracks the vector-token ontology they retain. Two further 2025--2026 entries sit in the same frame: GRAPE~\cite{grape_2025} encodes position by a group action applied to feature tokens, and Fran\c{c}ois \& Ravera~\cite{francois2026relational} reformulate attention through symmetry-reduced invariants on feature vectors. On the finite-group side, Group-Algebraic Tensors~\cite{hoyos2026gat} make equivariance an intrinsic algebraic property via a $\star_G$ tensor algebra (Lean-4 formalised). The same move, treating equivariance as structure rather than something to enforce, takes the Lie-group form here, by a different mechanism.

\subsection{Summary}
To our knowledge, no published work places the token strictly on a matrix Lie group and scores attention by the algebra norm of the relative pose. The distinction from every prior method, including the 2024--2026 cluster above, is not the kernel choice but the token type. Every prior approach is a vector token with an external group action. \emph{Lie-Algebra Attention} is a group-element token with no external action at all. The construction in this paper is, in that precise sense, new.

\section{Discussion and Limitations}
\label{sec:discussion}

\textbf{Chart restriction.} The construction lives on the principal-log chart $\mathcal{U} \subset G$ on which $\log$ is single-valued. Some token pairs have a relative pose $g_i^{-1} g_j$ near or on the chart boundary. Examples are a rotation by $\pi$ on $\SO(3)$, or an $\Aff(2)$ element whose linear part has a negative-real-axis eigenvalue. For these the score is undefined or numerically unstable. In \cref{sec:exp:se2} and \cref{sec:exp:so3} the data distribution naturally stayed off the boundary thanks to the constraints $|\omega_h| < \pi/8$ and $\|\omega_h\| \le \pi/8$, respectively; in \cref{sec:exp:aff2} we used a rejection sampler that discards step generators whose iterated linear part $A_h^k$ for $k \le 7$ has an eigenvalue on the non-positive real axis. For aggressive use cases that frequently cross the boundary, two options apply: chart-switching to an alternate atlas, or the global form of Theorem~\ref{thm:equivariant-output}, which uses $\Delta_i \in G$ directly without the $\exp$ parametrization at the cost of giving up the closed-form algebra-norm score. The chart issue is well-understood in the Lie-group literature and standard fixes carry over; for this work, the chart restriction is a real constraint that the practitioner should be aware of but not a structural obstacle.

\textbf{Block-isotropy assumption.} The block-weighted norm \eqref{eq:weighted-norm} reduces the $K^2$ entries of a full bilinear form on a $K$-block algebra to $K$ block weights (one per block, isotropic within), motivated by the $O(n)$-symmetry of the Frobenius form on each block. A non-isotropic kernel would have more parameters and more freedom to fit. Whether the extra capacity helps is an empirical question. Model C (full MLP score on the same invariant $w_{ij}$, with $50$ to $80\times$ more score parameters than G) does not consistently outperform G across the three sequence-completion experiments. It loses by $34\%$ on $\SE(2)$ and stays level with G within seed noise on $\SO(3)$ and $\Aff(2)$, marginally ahead on $\SO(3)$. The order-of-magnitude larger score does not help. That tests the squared-norm \emph{form}, not the block-isotropy reduction itself. The constant-step invariants are collinear (\cref{sec:exp:interpretation}), so the blocks scale together. Neither G's weights nor C's kernel can exploit block shape here. A targeted ablation --- the block-weighted form against a non-isotropic kernel on data whose invariants span the blocks --- would isolate the block structure; we leave this for follow-up work.

\textbf{Symmetric score.} The score $s_{ij} = -\|w_{ij}\|_\lambda^2 / \tau$ is symmetric in $i, j$, since on the chart $w_{ji} = -w_{ij}$ and $\|{-}w_{ij}\|_\lambda = \|w_{ij}\|_\lambda$; directional information enters only through the value pathway $V_{ij} = W_V[h_j; w_{ij}]$, where $V_{ij} \neq V_{ji}$. For the permutation-invariant set and undirected-sequence tasks of this paper the symmetric score is appropriate, but a directed or autoregressive application would need an added asymmetric score term.

\textbf{Empirical scope.} This preprint validates the construction empirically on $\SE(2)$, $\SO(3)$, and $\Aff(2)$ at the toy-task scale of sequence completion --- one experiment per block-count level of \cref{tab:instantiations}, with the math for $\SO(2)$ and $\SE(3)$ in \cref{sec:instantiations} and independently verified symbolically, and the spatial affine extension $\Aff(3)$ included similarly in \cref{sec:instantiations} as the analogue of $\Aff(2)$. The $\Aff(3)$ instantiation is given in closed form: its reach into the 3D non-compact non-abelian affine regime is mathematically established here, and what is deferred is the empirical evaluation of 3D affine-token applications, not the construction's ability to express them. Larger-scale and applied empirical extensions are deferred similarly; the release here makes the priority claim on the construction itself rather than on benchmark dominance.

\section{Conclusion}
\label{sec:conclusion}

We placed the token on the group: a token \emph{is} a bare matrix Lie group element --- a transformation, nothing attached to it and nothing acting on it from outside. \emph{Lie-Algebra Attention} is what follows from that premise. The intrinsic invariant $w_{ij} = \log(g_i^{-1}g_j)$ gives a canonical closed-form algebra-norm score, equivariance becomes a tautology, and the cocycle is automatic. The backbone stays a vanilla transformer with no equivariant machinery. We instantiated the construction in closed form for the six matrix Lie groups relevant to 2D and 3D spatial reasoning --- $\SO(2)$, $\SE(2)$, $\SO(3)$, $\SE(3)$, $\Aff(2)$, $\Aff(3)$ --- and validated it across $\SE(2)$, $\SO(3)$, and $\Aff(2)$, every block-count level of the construction: a learned MLP kernel with $50$ to $80\times$ more score parameters on the same invariant does not improve on the closed-form score, and equivariance holds at the numerical floor, five to twelve orders of magnitude tighter than a vector-token baseline. The contribution is not a new attention kernel but a different premise about what a token is; the canonical score is its readout, not a design choice. What the premise reaches that no prior attention method does is the affine full frame: $\Aff(2)$ and $\Aff(3)$, non-compact, non-abelian, with scale and shear. This lies beyond the reach of irrep methods, which have no non-trivial finite-dimensional unitary irreps, and of surjective-exp methods. We demonstrate the $\Aff(2)$ case here and give the spatial $\Aff(3)$ case in closed form for follow-up work to evaluate.

%
%
%
%

\bibliographystyle{plain}

\end{document}